\documentclass[12pt]{amsart} 

\usepackage{amssymb}
\usepackage{amsfonts}
\usepackage{amsmath}

\usepackage{epsfig}
\usepackage{amstext}
\usepackage{amsthm}
\usepackage{hyperref}
\usepackage{bm}
\usepackage{bbm}
\usepackage{eucal} 
\usepackage{nicefrac} 

\numberwithin{equation}{section}

\usepackage{a4wide}
\usepackage{hyperref}

\title[Kernel regression, minimax rates and effective dimensionality]{Kernel regression, minimax rates and effective dimensionality: beyond the regular case}

\author{Gilles Blanchard
        \and Nicole M\"{u}cke \\
        }

\address{Institute of Mathematics, University of Potsdam, Karl-Liebknecht-Straße 24-25 14476 Potsdam, Germany}
\email{\{blanchard,nmuecke\}@uni-potsdam.de}

\date{\today}

\theoremstyle{plain}

\newtheorem{theo}{Theorem}[section]

\newtheorem{lem}[theo]{Lemma}

\newtheorem{prop}[theo]{Proposition}

\theoremstyle{definition}

\theoremstyle{remark}

\newcommand{\cal}{\mathcal}

\newcommand{\F}{{\cal F}}
\newcommand{\G}{{\cal G}}

\newcommand{\EE}{{\mathbb{E}}}
\newcommand{\PP}{{\mathbb{P}}}
\newcommand{\RR}{{\mathbb{R}}}
\newcommand{\NN}{{\mathbb{N}}}

\newcommand{\lam}{\lambda}

\newcommand{\D}{\cal D}
\newcommand{\h}{{\cal H}}

\renewcommand{\H}{{\cal H}}
\newcommand{\X}{{\cal X}}
\renewcommand{\O}{{\cal O}}
\newcommand{\prf}{\begin{proof}} 
\newcommand{\prfend}{\end{proof}} 

\newcommand{\la}{\langle}
\newcommand{\ra}{\rangle}

\newcommand{\x}{{\bf x}}
\newcommand{\y}{{\bf y}}
\newcommand{\z}{{\bf z}}
\newcommand{\Z}{{\cal Z}}
\newcommand{\M}{{\cal M}}

\newcommand{\K}{{\cal K}}
\newcommand{\R}{{\cal R}}
\newcommand{\N}{{\cal N}}
\renewcommand{\P}{{\cal P}}
\newcommand{\eps}{\varepsilon}

\newcommand{\pfeil}{\longrightarrow}

\newcommand{\eigv}{\mu}
\newcommand{\eigf}{\psi}

\newcommand{\eigup}{{\bf ($\mbox{Eigen}^{<}(\nu^{*})$)}} 
\newcommand{\eiglow}{{\bf ($\mbox{Eigen}^{>}(\nu_{*})$)}} 

\newcommand{\SC}{{\bf SC($r,R$)}}

\newcommand{\bernstein}{ {\bf Bernstein($M, \sigma$)}}

\newcommand{\rad}{\pi}

\newcommand{\nux}{\nu}

\newcommand{\fo}{f_{\PP}^*}

\newcommand{\tr}{\mathrm{tr}}

\newcommand{\wh}{\widehat}

\newcommand{\paren}[1]{\left(#1\right)}

\newcommand{\norm}[1]{\left\|#1\right\|}

\newcommand{\set}[1]{\left\{#1\right\}}

\RequirePackage{color}

\newcounter{nbdrafts}
\makeatletter
\newcommand{\checknbdrafts}{
\ifnum \thenbdrafts > 0
\@latex@warning@no@line{**********************************************************************}
\@latex@warning@no@line{* The document contains \thenbdrafts \space draft note(s)}
\@latex@warning@no@line{**********************************************************************}
\fi}

\makeatother

\begin{document}

\maketitle

\begin{abstract}
  We investigate if kernel regularization methods can achieve minimax convergence rates
  over a source condition regularity assumption for the target function.
  These questions have been considered in past literature, but only under specific assumptions
  about the decay, typically polynomial, of the spectrum of the
  the kernel mapping covariance operator. In the perspective of distribution-free results,
  we investigate this issue under much weaker assumption on the eigenvalue decay, allowing
  for more complex behavior that can reflect different structure of the data at different scales.
\end{abstract}

\section{Introduction and motivation}

\subsection{Kernel regression}

We are concerned with the classical regression learning problem,
where we observe training data $\D:=(X_i,Y_i)_{i=1,\ldots,n}$\,, assumed to be an i.i.d. sample
from a distribution $\PP$ over $\X \times \RR$\, ( $\nux$ will denote the marginal
distribution of $X$), and the goal is to estimate
the regression function $f^*(x) :=\EE[Y|X=x]$\,. We consider the well-established setting
of (reproducing) kernel learning: we assume a positive semi-definite kernel $k(.,.)$ has
been defined on $\X$\,, with associated canonical feature mapping $\Phi: \X  \rightarrow \H$ 
into a corresponding reproducing kernel Hilbert space $\H$\,. A classical methodology
to handle this problem is kernel ridge regression, depending on a regularization parameter
$\lam>0$\,, giving rise to the estimate $\wh{f}_\lam$\,; 
we will consider a much larger class of kernel methods below, but restrict ourselves
to that particular example for the sake of the present introduction.

A common goal of learning theory is to give upper bounds for the convergence of $\wh{f}_\lam$
to $f^*$\,, and derive rates of convergence (as $n\rightarrow \infty$) under appropriate
assumptions on the ``regularity'' of $f^*$\,.
In this paper, te notion of convergence we will
consider is the usual squared $L^2(\nux)$ distance with respect to the sampling distribution,
$\|\wh{f}_\lam-f^*\|^2_{2,\nux} = \EE\big[\big( \wh{f}_\lam(X) - f^*(X) \big)^2\big]$\,;
it is well-known
that this is equal to the excess risk with respect to Bayes when using the squared loss, i.e.
\[
\|\wh{f}_\lam-f^*\|^2_{2,\nux} = \EE\big[\big( Y - \wh{f}_\lam(X)\big)^2\big]  -
  \min_{f: \X \rightarrow \RR} \EE\big[\big( Y -  f(X)\big)^2\big]\,.
 \]
 More precisely, we are interested in bounding the averaged above error over the
 draw of the training data (this is also called Mean Integrated Squared Error or MISE in the
 statistics literature). We will also consider convergence estimates in the (stronger)
 averaged $\H$-norm.

\subsection{Minimax error in classical nonparametrics}
 
When upper bounds or convergence rates for a specific method are obtained, it is natural to ask
to what extent they can be considered optimal.
The classical yardstick is the notion of minimax error over a set of candidate
$\P$ (hypotheses) for the data generating distribution:
\begin{equation}
\R(\P,n) := \inf_{\wh{f}} \sup_{\PP \in \P} \EE_{\D \sim \PP^{\otimes n}}
      \big[ \|\wh{f}-f^*_{\PP}\|^2_{2,\nux} \big] \,,
\end{equation}
where the $\inf$ operation is over all estimators, and we added an index $\PP$ to
$f^*$ to emphasize its dependence on the data generating distribution.

In the nonparametric statistics literature, it is commonly assumed that $\X$ is some compact
set of $\RR^d$\,, the sampling distribution $\nux$ has an upper bounded density with respect
to the Lebesgue measure, and the type of regularity considered for the target function
is a Sobolev-type regularity, i.e., that the target function $f^*$ has squared-integrable $r$-th
derivative. This is equivalent to say that $f^*$ belongs to a Sobolev ellipsoid of radius
$R$\,,
\begin{equation}
\label{eq:sobolev1}
f^* \in \Big\{ f: \sum_{i\geq 1} i^{-\frac{2r}{d}} f_i^2 \leq R^2 \Big\}\,,
\end{equation}
where $f_i$ denotes the coefficients of $f$ in a (multidimensional) trigonometric function basis.
Minimax rates in such context are known to be of the order $\O(n^{-\frac{2r}{2r+d}})$\,,
and can be attained for a variety of classical procedures \cite{Sto82,tsybakov}.

\subsection{Minimax error in a distribution-free context}

In the statistical learning context, the above assumptions are unsatisfying. The first reason is that
learning using kernels is often applied to non-standard spaces, for instance graphs,
text strings, or even probability distributions (see, e.g., \cite{CriSha04}). There is often no ``canonical'' notion of
regularity of a function on such spaces, nor a canonical reference measure
which would take the role of the Lebesgue measure in $\RR^d$\,. The second reason is that learning 
theory focuses on a distribution-free approach, i.e., avoiding specific assumptions on the generating
distribution. By contrast, it is a very strong assumption to posit that the sampling distribution $\nux$ is
dominated by some reference measure (be it Lebesgue or otherwise), especially for non-standard spaces,
or in $\RR^d$ if the dimension $d$ is large. In the latter case, the convergence rate 
$\O(n^{-\frac{2r}{2r+d}})$ becomes very slow (curse of dimensionality), yet it is often noticed in practice
that many kernel-based methods perform well. The reason is that for high-dimensional data, more often
than not the sampling distribution $\nux$ is actually concentrated on some lower-dimensional structure,
so that the assumption of $\nux$ having bounded density in $\RR^d$ is violated: convergence rates
could then be much faster. For these reasons, it has been proposed 
to consider regularity classes for the target function having a form similar to \eqref{eq:sobolev1}, 
but reflecting implicitly the geometry corresponding to the choice of the kernel and to the sampling distribution.
More precisely, denote $B_\nu$ (uncentered) covariance
operator of the kernel feature mapping $\Phi(X)$ and $(\eigv_{\nu,i},\eigf_{\nu,i})_{i\geq 1}$ an eigendecomposition of $B_\nu$\,. 
For $r,R>0$\,, introduce the class
\begin{equation}
\label{eq:source}
\Omega(\nu,r,R) := \Big\{ f \in \H: \sum_{i\geq 1} \eigv_{\nux,i}^r f_i^2 \leq R^2 \Big\} = B_\nu^r B(\H,R) \,,
\end{equation}
where $B(\H,R)$ is the ball of radius $R$ in $\H$\,, and $f_i:=\langle f, \psi_{\nux,i} \rangle$ are the coefficients
of $f$ in the eigenbasis. Clearly, \eqref{eq:source} has a form similar to \eqref{eq:sobolev1}, but
in a basis and scaling that reflects the properties of the distribution of $\Phi(X)$\,.
If the target function $f^*$ is well approximated in this basis in the sense that its coefficients
decay fast enough in comparison to the eigenvalues, it is considered as regular in this
geometry (higher regularity corresponds to higher values of $r$ and/or lower values of $R$).
This type of regularity class, also called {\em source condition}, have been considered in a
statistical learning context by \cite{cusmale}, and \cite{optimalratesRLS} have established upper bounds for
the performance of kernel ridge regression $\wh{f}_\lam$ over such classes; this has been extended to other
types of kernel regularization methods by \cite{optimalrates,DicFosHsu15,BlaMuc16}. These bounds rely on 
tools introduced in the seminal work of \cite{Zha05}, and depend in particular on the notion
of {\em effective dimension} of the data with respect to the regularization parameter $\lam$\,,
defined as
\begin{equation}
\label{def:eff_dim}
\N(\lam) := \tr( (B_\nu + \lam)^{-1} B_\nu) = \sum_{i\geq 1} \frac{\eigv_{\nux,i}}{\eigv_{\nux,i} + \lam}\,. 
\end{equation}
As before, the next question of importance is whether such upper bounds can be proved to
be minimax optimal over the classs $\Omega(\nu,r,R)$\,, assuming the regularization parameter
$\lam$ is tuned appropriately. This question has been answered positively when
a polynomial decay of the eigenvalues, $\mu_{v,i} \asymp i^{-b}$\,, is assumed ($\asymp$ stands for upper and lower bounded up to a constant). In this case $\N(\lam)$ can be evaluated,
and for an appropriate choice of $\lam$\,, the upper bound can be matched by a corresponding
lower bound. The first comprehensive result in this direction was established by
\cite{optimalratesRLS}\,; \cite{BlaMuc16} give a sharp estimate of the convergence rate in this case
including the dependence on the parameters $R$ and noise variance $\sigma$\,, namely
$\O\Big( R^2 \big(\frac{\sigma^2}{R^2n}\big)^{\frac{2r+1}{2r+1+\nicefrac{1}{b}}} \Big)\,.$

\subsection{Contributions: beyond regular decay of eigenvalues}

The assumption of polynomially decaying eigenvalues yields explicit minimax
convergence rates, relates closely \eqref{eq:source} to \eqref{eq:sobolev1}, and ensures
that kernel methods can achieve those optimal rates. Yet, once again it
is unsatisfying from a distribution-free point of view. Remember the stucture of the eigenvalues
reflects the covariance of the feature mapping $\Phi(X)$; for complex data, there is no
strong reason to expect that their decay should be strictly polynomial, and we would like to
cover behavior as general as possible for the eigenvalue decay, for instance:
\begin{itemize}
\item decay rates including other slow varying functions, such as $\mu_{v,i} \asymp i^{-b} (\log i)^c (\log \log i)^d$\,;
\item eigenvalue sequence featuring plateaus separated by relative gaps;
\item shifting of switching along the sequence between different polynomial-type regimes,
\end{itemize}
which all might correspond to diffent type of structure of the data at different scales.
With the distribution-free point of view in mind, we therefore try to characterize minimax rates
for target function classes of the form \eqref{eq:source}, striving for assumptions as weak as possible
on the eigenvalue sequence. More importantly, we investigate if kernel methods also achieve minimax rates
in this general case.



\section{Setting}

We let $\Z = \X \times \R$ denote the sample space, where the input space $\X$ is a standard Borel space endowed with a fixed unknown probability measure $\nux$.  
The kernel space $\h$ is assumed to be separable, equipped with a measurable positive semi-definite kernel $k$, bounded by $\kappa$, implying 
continuity of the inclusion map $I_{\nu} : \h \pfeil L^2(\nu)$. Moreover, we consider the covariance operator 
$B_{\nu}=I^*_{\nu}I_{\nu}=\EE [ \Phi (X) \otimes \Phi (X) ] $, which can be shown to be positive self-adjoint trace class (and hence is compact). 
Given a sample $\x=(x_1, \ldots, x_n) \in \X^n$, we define the sampling operator $S_{\x}: \h \pfeil \R^n$ by $(S_{\x}f)_i= \la f, \Phi(x_i) \ra_{\h}$. The empirical 
covariance operator is given by $B_{\x}=S_{\x}^*S_{\x}$. Throughout we denote by $\eigv_j$ the positive eigenvalues of $B_{\nu}$ in decreasing order, satisfying $0 < \eigv_{j+1} \leq \eigv_{j}$ for all $j>0$ and 
$\eigv_j \searrow 0$. For any $t>0$ we let    
\begin{equation}
\F (t):= \#\{j \in \NN: \mu_j \geq t\} \;. 
\end{equation}
Note that $\F$ is left-continuous and decreasing as $t$ grows with $\F(t)\equiv 0$ for any $t > \kappa^2$\,, and $\F$ has limit $+\infty$
in $0^+$. Given $r>0$, 
we set $\G(t):= \frac{t^{2r+1}}{\F(t)}$ (possibly taking the value $\infty$ if $\F(t)=0$), which is left-continous and increasing on $(0,\kappa^2]$ with $\G(0^+) =0$.
Define the generalized inverse for any $u>0$ by
\begin{equation}
    \label{eq:definverse}
  \G^{-1}(u) := \max\set{t: \G(t) \leq u}\,.
\end{equation}
Some properties of $\F, \G$ and $\G^{-1}$ are collected in Lemma \ref{lem:prop_FG}. 

In order to present our main results, we require some assumptions on the learning problem:   

(1) Noise Assumption: The sampling is assumed to be random i.i.d., where each observation point $(X_i,Y_i)$ follows the model
$ Y = f(X) + \eps \,.$ We will make the following Bernstein-type assumption on the observation noise distribution: For any integer $m \geq 2$ 
and some $\sigma > 0$ and $M>0$: \\
{\bernstein: $\EE[\; \eps ^{m} \; | \; X \;] \leq \frac{1}{2}m! \; \sigma^2 M^{m-2}$ \; $\nux - {\rm a.s.}$\,.

(2) Assumption on eigenvalues: For any $j$ sufficiently large and some $\nu_{*}\geq \nu^{*}\geq 1$ \\

\eigup :  $ \eigv_{2j} / \eigv_j  \leq 2^{-\nu^{*}}$,  \qquad  \eiglow :  $ \eigv_{2j}/ \eigv_j \geq 2^{-\nu_{*}} $ \; . \\

(3) The regularity  of the target function $\fo$ is measured in terms of a source condition:\\ 
\SC : $\fo \in B_\nu^r B(\H,R)$, where $r>0$, $R>0$. 

(4) Model: Let $\theta=(M, \sigma, R) \in \R^3_+$. The class $\M_{\theta}:=\M(\theta , r, \nu)$ consists of all distributions $\PP$ with fixed 
$X$-marginal $\nu$ and conditional distribution of $Y$ given $X$
satisfying \bernstein\, for the deviations\, and \SC \, for the mean.  

(5) Regularization: We find our estimator for $\fo$ via some linear spectral regularization function $g_{\lam}: [0, 1] \longrightarrow \R$, 
satisfying the following conditions for any $0 < \lam \leq 1$:
\[  \sup_{0<t\leq 1}|tg_{\lam}(t)| \leq D\; , \quad \sup_{0<t\leq 1}|g_{\lam}(t)| \leq \frac{E}{\lam}, \quad \sup_{0<t\leq 1}| r_{\lam}(t)  | \leq \gamma_0 \;,  \]
where $r_{\lam}(t)=1-g_{\lam}(t)t$, for some $D<\infty$, $E<\infty$ and $\gamma_0 <\infty$. The  qualification of the regularization $\{g_{\lam}\}_{\lam}$ is the maximal $q$ such that for any $0<\lam\leq 1$
\begin{equation*}
\sup_{0<t\leq 1}|r_{\lam}(t)|t^{q} \leq \gamma_{q}\lam^{q},
\end{equation*}
for some constant $\gamma_q>0$\,. 

These conditions originally characterize regularization methods for finding stable solutions for ill-posed inverse problems (\cite{engl}). It has been shown in e.g. \cite{rosasco} and \cite{learning} that precisely these methods are very well applicable in the context of learning theory.

Given a sample $\z=(\x, \y)\in (\X \times \R)^n$, we define the estimator $f_{\z}^{\lam}$ for a suitable a-priori parameter choice $\lam = \lam_n$  by 
\begin{equation}
\label{estimator}
f_{\z}^{\lam_n}:=  g_{\lam_n}(\kappa^{-2} B_{\x})\kappa^{-2} S_{\x}^{\star}\y =
g_{\lam_n}(\bar B_{\x})\bar S_{\x}^{\star}\y \; ,
\end{equation}
where we have introduced the shortcut notation $\bar B_{x} := \kappa^{-2} B_{\x}$
and $ \bar S_{\x} := \kappa^{-2} S_{\x}$\,.


\section{Main results}

Under the assumptions of the previous section, we establish that the order of the minimax rates is  given by the following quantities:

\begin{equation}
\label{rateseq}
 a^{\h}_{n,\theta} := R \G^{-1}\left(\frac{\sigma^2}{R^2n} \right)^r\; ,  \qquad  
 a^{L^2(\nu)}_{n,\theta} := R \G^{-1}\left(\frac{\sigma^2}{R^2n} \right)^{r + \frac{1}{2}}
\end{equation}

\begin{theo}[Minimax lower rate]
\label{minimaxlowerrate}
Let $r>0$ be fixed and assume $\nu$ satisfies \eiglow . Then the sequence  $(a^{\h}_{n,\theta})_n$ defined in \eqref{rateseq} is a  minimax lower rate of convergence  
for the RKHS-norm $\h$, for the model family $\M_{\theta} := \M(\theta , r)$\,, $\theta :=(M,\sigma , R) \in \R_+^3$, i.e. 
there exists $\tau > 0$ (possibly depending on $r$ and $\nu_{*}$, but not on $\theta$) such that
\begin{equation}
\label{minilow}
 \liminf_{n \to \infty} \inf_{f_{\bullet}}\sup_{\PP \in {\M_\theta}}
 \PP^{\otimes n}(\;  || \fo - f_{\z} ||_{\h} > \tau a^{\h}_{n,\theta} \;) > 0 , 
\end{equation}     
where the infimum is taken over all estimators. Similarly, $\eqref{minilow}$ also holds w.r.t. to $L^2(\nu)$-norm with minimax lower rate $a^{L^2(\nu)}_{n,\theta}$ 
over the same model family. 
\end{theo}

\begin{theo}[Upper rate]
\label{maintheo3}
Assume \SC \,, \bernstein \, and that $\nu$ satisfies \eigup \,. Consider the model $\M_{\theta}$, where 
$\theta:=(M,\sigma ,R) \in \R^3_+$.  
Given a sample $\z \in \Z^n$, define $f_{\z}^{\lam}$ as in \eqref{estimator},
using a regularization function of qualification $q \geq r $, with the parameter sequence
\begin{equation}
  \label{paramrule}
   \lam_{n, \theta} = \min\paren{ \G^{-1}\left(\frac{\sigma^2}{R^2n} \right) ,1} \; .
\end{equation}
Then there exists $n_0>0$ (depending on the above parameters) such that for any $n \geq n_0$ 
with probability at least $1-\eta$ \,:
\begin{equation}
\label{eq:reduxed_new}
   ||\fo - f_{\z}^{\lam_n}||_{\h}  \leq   C_{r, \nu^*} \; \log(8\eta^{-1})\, a^{\h}_{n,\theta}   \, ,  
\end{equation}
provided $\log{(8\eta^{-1})}\leq C_{r, \nu^*}\lam_n^{-r}$. 
Similarly, the sequence $ a^{L^2(\nu)}_{n,\theta}$ also fulfills $\eqref{eq:reduxed_new}$ in $L^2({\nu})$-norm with 
identical parameter sequence $\eqref{paramrule}$, using a regularization function with qualification $q\geq r+\frac{1}{2}$.
\end{theo}

{\bf Remark:} In the above Theorem, $n_0$ can possibly depend on all parameters\,,
but the constant $C_{r, \nu^*}$ in front of the upper bound does not depend on $R,\sigma$, nor $M$\,. In this sense,
this result tracks precisely the effect of these important parameters on the scaling of the rate,
but remains asymptotic in nature: it cannot be applied if, say, $R,\sigma$ of $M$ also depend on $n$
(because the requirement $n\geq n_0$ might then lead to an impossiblity.) 


 

\section{Discussion}

{\bf Range of applications}. The assumptions we made on the spectrum decay, namely \eigup and \eiglow,
are much weaker than the usual assumptions of polynomial decay on the eigenvalues. Therefore, our results establish that
in this much broader situation, usual kernel regularization methods can achieve minimax rates over the regularity
classes \SC. In particular, these conditions accomodate for changing behavior of the spectrum at different scales,
as well as other situations delineated in the introduction. Still, our conditions do not encompass totally
arbitary sequences: \eiglow in particular implies that the eigenvalues cannot decrease with a polynomial rate with 
exponent larger than $\nu_*$\,. While the latter constant can be chosen arbitrary large 
and only results in a change of constant factor in the rates,
it excludes for example exponentially decreasing eigenvalues. It is not clear whether such sharp results
about classes \SC can be established in this case. However, we observe that in such cases (typically observed when using the 
Gaussian kernel), in practice the kernel parameters (e.g. the bandwith) are also tuned in addition to $\lambda$\,,
which might reflect the fact that for badly tuned bandwidth of the kernel, tuning of the regularization parameter
$\lambda$ alone might not give satisfactory (minimax) results. On the other hand, the results we obtained might provide an
additional motivation for using ``rougher'' kernels than Gaussian, leading to a softer decay of eigenvalues,
in which case minimax adaptivity is at hand over a large regularity class. The latter is true however only if the qualification of the method
is large, which is not the case for the usual kernel ridge regression: hence, rougher kernels should be used with
methods having a large qualification (for instance $L^2$ boosting).

{\bf Adaptivity}. Our results establish the existence of a suitable regularization parameter $\lambda$ such that the
associated estimator attains the minimax rate if the regularity class parameter are known in advance. The latter is of course
not realistic, but in the case of $L^2$ (prediction) error, the principle of using a grid for $\lambda$ and then using
a hold-out sample to select a value of $\lambda$ from the data is known to be able to select a value close to the optimal
choice  in a broad domain of situations (see, for instance, \cite{CapYao10}), so that we can turn our results to 
data-dependent minimax adaptivity even in the absence of a priori knowledge of the regularity parameters.


\section{Proofs}

The proof of Theorem \ref{maintheo3} and Theorem \ref{minimaxlowerrate} will be given not only in both $\h$- norm and $L^2(\nu)$-norm, but also for all intermediate norms, namely 
for $|| f||_s := ||B^s f||_{\h}$, where $s \in [0, 1/2]$. Note that $s=0$ corresponds to $\h$- norm, while $s=1/2$ corresponds to $L^2(\nu)$-norm. 

To simplify notation we will adopt the following conventions: The dependence of multiplicative constants $C$ on various parameters will
   (generally) be omitted, except for $\sigma, M, R,\eta$ and $n$ which we want to track precisely\,.
The expression ``for $n$ big enough'' means that the statement holds for
   $n\geq n_0$\,, with $n_0$ potentially depending on all model parameters
   (including $\sigma, M$ and $R$), but not on $\eta$\,.


\subsection{Preliminaries}

We start by collecting some useful properties for the functions
$\F$ and $\G$ in the following lemma. 

\begin{lem}
\label{lem:prop_FG}
\begin{enumerate}
    \item Let $c\leq 1$ be fixed. Then for any $t$\,,
      \[
      \G(ct) \leq c \G(t)\,.
      \]
  \item Assume \eigup holds.  Let $C\geq 1$ be fixed. Then for any $t$ small enough,
    \[
    \F(t) \leq 4C^{\frac{1}{\nu^*}} \F(Ct) \qquad \text{ and }
    \qquad
      \G(Ct) \leq 
    4C^{2r+1+\frac{1}{\nu^*}}\G(t)\,.
    \]
  \item Assume \eigup holds. For any $u>0$ it holds $\G\big( \G^{-1}(u)) \leq u$ and for $u$ small enough,
    \[
     \G\big( \G^{-1}(u)) \geq  \frac{u}{4}\,.
    \]
  \end{enumerate}
\end{lem}

\begin{proof}[Proof of Lemma \ref{lem:prop_FG}]
For point 1 of  the Lemma, let $c\leq 1$ be fixed; just write by definition of $\G$ and the fact
  that $\F$ is nonincreasing:
  \[
  \G(ct) = \frac{c^{2r+1}t^{2r+1}}{\F(ct)} \leq c^{2r+1} \G(t) \leq c \G(t)\,.
\]
  
  For point 2, let $j_0 \geq 1$ such that \eigup \; holds for any $j\geq j_0$
  and let $t_0$ be small enough such that $\F(t_0) \geq j_0$\,.

  Let $C\geq 1$ be fixed and $t \leq C^{-1}t_0$, so that
  $k_t:=\F(Ct) \geq j_0$\,. By definition $\eigv_{k_t+1} < Ct$\,.
  Furthermore for any $i\geq 1$ we have $\eigv_{2^i(k_t+1)} \leq 2^{-i\nu^*} \eigv_{(k_t+1)}$  
  by repeated application of \eigup. Choosing $i:= 1+ \lfloor \frac{\log_2 C}{\nu^*} \rfloor$,
  we have $2^{-i\nu^*}C \leq 1$ and $2^i\leq 2C^{\frac{1}{\nu^*}}$\,. Combining the first inequality
  with what precedes we deduce $\eigv_{2^i(k_t+1)} < t$ and thus
  $\F(t) \leq 2^i(k_t+1) -1 \leq 4C^{\frac{1}{\nu^*}} \F(Ct)$\,.
  We deduce
  \[
  \G(Ct) = \frac{C^{2r+1}t^{2r+1}}{\F(Ct)} \leq 
    \frac{4C^{2r+1+\frac{1}{\nu^*}}t^{2r+1}}{\F(t)} = 4C^{2r+1+\frac{1}{\nu^*}}\G(t)\,.
    \]

    We turn to point 3. Since $\G$ is left-continuous, the
    supremum in the definition (2.2) of its inverse $\G^{-1}$  is indeed
    a maximum (also the set over which the max is taken is nonempty since $\G(0^+)=0$ and $u>0$),
    and therefore must satisfy $\G(\G^{-1}(u))\leq u$.

    Consider now $\F(t^+):=\#\set{j \in \NN: \eigv_i > t}$,
    let $t_0'$ be small enough such that $\F(t_0') \geq 2 j_0$\,,
    and assume $t < \min(t'_0,\eigv_1)$. The second
    component of the latter minimum ensures $\F(t^+)\geq 1$.
    If $t \not\in \set{\eigv_i, i\geq 1}$, then $\F$ is continous in $t$
    and $\F(t) = \F(t^+)$. Otherwise, $t = \eigv_k$ with $k=\F(t) \geq 2j_0$\,,
    so that $\frac{\eigv_{k}}{\eigv_{\lfloor{\frac{k}{2}\rfloor}}} \leq 2^{-\nu^*}<1$\,,
    that is to say $t < \eigv_{\lfloor{\frac{k}{2}\rfloor}}$\,,
    implying $F(t^+) \geq \lfloor{\frac{k}{2}\rfloor} \geq \frac{1}{2} \F(t) -1$
    and finally $\F(t) \leq 4 \F(t^+)$\,.

    Consider now $u$ small enough such that $t=\G^{-1}(u) < \min(t'_0,\eigv_1)$ as above.
    Then $\G(t^+) \geq u$ and
      \[
      \G(\G^{-1}(u)) = \G(t)  = \frac{t^{2r+1}}{\F(t)}
      \geq \frac{1}{4} \frac{t^{2r+1}}{\F(t^+)}
      = \frac{1}{4} {\G(t^+)} \geq \frac{u}{4}\,.
      \]

\end{proof}

\begin{lem}[Effective dimensionality]
\label{lem:effect_dim}
Assume  \eigup \, holds. The effective dimension (defined in \eqref{def:eff_dim}) satisfies for any $\lam$ sufficiently small and some $C_{\nu^{*}}<\infty$
\begin{equation}
\N (\lam)  \leq C_{\nu^{*}} \F(\lam) \; .
\end{equation}
\end{lem}

\begin{proof}[Proof of Lemma \ref{lem:effect_dim}]
  Let $j_0 \geq 1$ such that \eigup  \; holds for any $j\geq j_0$ and let $\lam_0$ be small enough such that
  $\F(\lam_0) \geq j_0$\,. For $\lam \leq \lam_0$, denote $j_\lam := \F(\lam)$. Then,
  using $\eigv_j < \eigv_j + \lam$ and $\lam < \eigv_j + \lam $, we obtain
\begin{align*}
\N (\lam)  
= \sum_{j=1}^{j_\lam} \frac{\eigv_j}{\eigv_j + \lam} + 
\sum_{j> j_\lam} \frac{\eigv_j}{\eigv_j + \lam} 
\leq j_\lam + \frac{1}{\lam} \sum_{j > j_\lam}\eigv_j  \; .
\end{align*}
Focussing on the tail sum we see that
\begin{align*}
\sum_{j \geq j_\lam}\eigv_j  = \sum_{l=0}^{\infty} \sum_{j=(j_\lam+1)2^l}^{(j_\lam+1)2^{l+1}-1} \eigv_j   
\leq  (j_\lam+1) \eigv_{j_\lam+1} \sum_{l=0}^{\infty} 2^l2^{-l\nu{*}}
&=  (j_\lam+1) \eigv_{j_\lam+1} (1-2^{1-\nu{*}})^{-1}\\  &\leq \lam(\F(\lam) +1) (1-2^{1-\nu{*}})^{-1}
\,, \nonumber
\end{align*}
where the first inequality comes from the fact that the sequence $(\eigv_j)_j$ is decreasing
and by repeated application of \eigup; and the second inequality from the definition
of $j_\lam$.
Collecting all ingredients
we obtain for any $\lam \leq \lam_0$:
\begin{align*}
  \N (\lam) & \leq \F(\lam)(1+2(1-2^{1-\nu{*}})^{-1})\,.
\end{align*}
\end{proof}


\subsection{Proof of upper rate}

The proof of Theorem \ref{maintheo3} relies on the following non-asymptotic result, given in \cite{BlaMuc16}:

\begin{prop}
\label{maintheo1}
Let $s\in[0,\frac{1}{2}]$, assume \SC \, and  \bernstein \,. 
Let $f_{\z}^{\lam}$ be defined as in \eqref{estimator} using
a regularization function of qualification
$q \geq r+s$\,. 
Then, for any $\eta \in (0,1)$, $\lam \in (0, 1]$ and $n \in \NN$ satisfying
\begin{equation}
\label{assumpt1}
  n \geq 64 \lam^{-1} \max({\cal N}(\lam),1) \log^2{\left(8 / \eta\right)}  \; ,
\end{equation}  
we have  with probability at least 
$1-\eta$:
\begin{equation}
\label{maintheo1eq}
\norm{B_{\nu}^s (\fo - f_{\z}^{\lam})}_{\h}  \leq C_{r,s, \kappa} \; \log(8\eta^{-1})  \lam^s \left(  R\paren{\lam^r +   \frac{1}{\sqrt n}}  
+  \left( \frac{M}{n\lam} + \sqrt{\frac{\sigma^2 {{\cal N}(\lam)}}{{n \lam}}}\right) \right)\,.
\end{equation}
\end{prop}

\begin{proof}[Proof of Theorem \ref{maintheo3}:]
Let all assumptions of Theorem \ref{maintheo3} be satisfied. 
For proving \eqref{eq:reduxed_new} we want to apply Proposition \ref{maintheo1}\;. Provided n is big enough, we have $\F(\lam_n)\geq 1$ and by Lemma \ref{lem:effect_dim} it holds $\N(\lam_n) \leq C_{\cdot} \lam_n^{2r+1}/\G(\lam_n)$, following from the definition of $\G$. By the definition of $\lam_n$ and by Lemma \ref{lem:prop_FG}, $(iii)$, for n sufficiently large,    $\G(\lam_n)\geq C_{\sigma, R}\frac{1}{n}$, so assumption \eqref{assumpt1} is satisfied if 
$\log{(8/\eta)}\leq C_{\sigma, R}\lam_n^{-r}$. Hence, with probability at least $1-\eta$
\begin{equation}
\label{maintheo1eqredux}
 \norm{B^s (\fo - f_{\z}^{\lam_n})}_{\h}  \leq C_{\cdot} \; \log(8\eta^{-1}) \lam_n^s \left( R\paren{\lam_n^r +   \frac{1}{\sqrt n}}  
+ \left( \frac{M}{n\lam_n} +\sigma \sqrt{\frac{\lam_n^{2r}}{n\G(\lam_n)}}
 \right)\right)\,.
\end{equation}
Observe that the choice \eqref{paramrule} 
implies that $n^{-\frac{1}{2}} = o(\lam_n^r)$, since $\sigma^2/R^2n=\G(\lam_n) \leq \lam_n^{2r +1}$. Therefore, up to requiring $n$ large enough
and multiplying the front factor by 2 ,
we can disregard the term  $1/\sqrt{n}$ in the second factor of the above bound. 
Similarly, one can readily check that
\[
\frac{M}{n\lam_n} = o\paren{ \sqrt{\frac{\lam_n^{2r}}{n\G(\lam_n)}} }\,,
\]
so that we can also disregard the term $(n\lam_n)^{-1}$ for $n$ large enough 
and concentrate on the two remaining main terms of the upper bound in \eqref{maintheo1eqredux}, 
which are $R \lam_n^r$ and $\sigma\lam_n^{r}\G(\lam_n)^{-\frac{1}{2}} n^{-\frac{1}{2}}$ \,.
The proposed choice of $\lam_n$ balances precisely these two terms and 
easy computations lead to \eqref{eq:reduxed_new}.
\end{proof}


\subsection{Proof of minimax lower rate}

Let $r>0$, $R>0$ and $s \in [0,1/2]$ be fixed. Assume \eiglow \;.  
In order to obtain minimax lower bounds, we apply a general reduction scheme as presented in Chapter 2 of \cite{tsybakov}. The main idea is to find $N_{\eps}$ functions $f_1, \ldots, f_{N_{\eps}}$ satisfying \SC \,, depending on $\eps$ sufficiently small, with $N_{\eps} \to \infty$ as $\eps \to 0$, such that any two of these functions are $\eps$-separated with respect to $||\cdot||_{s}$- norm, but such that the associated distributions $\PP_{f_j} =: \PP_j \in \M$ (see definition \eqref{stochkern} below) have small Kullback-Leibler divergence $\K$ to each other and are therefore statistically close. It is then clear that 
 \begin{align}
\label{reduction}
\inf_{f_{\bullet}}\sup_{\PP \in {\M}} 
\PP [ \norm{ B_{\nu}^s( f_{\PP}^* - f_{\z} )} _{\h} \geq \eps ] 
&\geq   \inf_{f_{\bullet}}\max_{1\leq j \leq N_{\eps}} 
\PP_{j}[ \norm{ B_{\nu}^s( f_j - f_{\z}) } _{\h}^p \geq \eps ]
\;,
\end{align}
 where the infimum is taken over all estimators $f_{\bullet}$ of $\fo$. The above RHS is then
 lower bounded through { Proposition 6.1 given in supplementary material} which is a consequence of Fano's lemma.
 
We will establish the lower bounds in the particular case where the distribution of $Y$ given $X$ is Gaussian with variance $\sigma^2$ (which satisfies \bernstein \, with
$M=\sigma$)\,. The main effort is to construct a finite subfamily belonging to the model of interest and suitably satisfying the assumptions of {Proposition \ref{Fano} in supplementary}. More precisely, to each $f$ satisfying \SC \, and $x \in \X$ we associate the following measure:
\begin{equation}
\label{stochkern}
\PP_f(dx,dy):=\PP_f(dy|x)\nux (dx)\,, \mbox{ where } \PP_f(dy|x):= {\N} (f(x),\sigma^2)\,.
\end{equation}
Then the measure $\PP_f$ belongs to the class $\M_{(M, \sigma, R)}$, defined in Section 2,  
and given $f_1, f_2$ satisfying \SC \,, the Kullback divergence between $\PP_1$ and $\PP_2$ satisfies
\begin{equation}
\label{eq:kull_stochkern}
 \K(\PP_1, \PP_2)  = \frac{1}{2 \sigma^2} \big\| \sqrt{B_{\nu}}(f_1 - f_2) \big\|_{\h}^2 \,. 
 \end{equation}
 
 For the sake of completeness, we give a constitutive result whose proof is given in \cite{tsybakov}.  
\begin{prop}
\label{Fano}
Consider a model ${\cal P}=\{P_{\theta}: \theta \in \Theta\}$ of probability measures 
on a measurable space $(\Z, {\cal A})$\,, indexed by $\Theta$. Additionally, let $d: \Theta \times \Theta \longrightarrow [0, \infty)$ be a (semi-) distance. 
Assume that $N \geq 2$ and suppose that $\Theta$ contains N+1  elements $\theta_0, \ldots, \theta_N$ such that:
\begin{enumerate}
\item[(i)] For some $\eps>0$\,, and for any $0 \leq i < j \leq N$\,,
$d(\theta_i, \theta_j) \geq 2\eps$ \; ;
\item[(ii)] For any $j=1,\ldots,N$\,,
$P_j$ is absolutely continuous with respect to $P_0$,  and
\begin{equation}
\label{max}
 \frac{1}{N}\sum_{j=1}^{N} {\cal K}(P_j, P_{0}) \leq \omega  \; \log(N)  \; , 
\end{equation}
for some $0 < \omega < 1/8$.  
\end{enumerate}
Then
\begin{equation*}
\inf_{\hat \theta}\max_{1\leq j \leq N} P_{j}(\; d(\hat \theta, \theta_j) \geq \eps \;) \;  \geq \; \frac{\sqrt{N}}{1+\sqrt{N}}
   \left( 1-2\omega - \sqrt{\frac{2\omega}{\log(N)}} \right) > 0 \; ,
\end{equation*}
where the infimum is taken over all estimators $\hat \theta$ of $\theta$.
\end{prop}

We will need the following Proposition:

\begin{prop}
\label{prop:kullback}
For any $\eps>0$ sufficiently small (depending on the parameters $\nu_{*}, r, R, s$), there exist $N_{\eps} \in \NN$ and functions $f_1,\ldots, f_{N_{\eps}} \in \Omega(\nu, r,R)$ satisfying
\begin{enumerate}
\item 
For any $i,j = 1,\ldots,N_{\eps}$ with $i \not = j$ it holds $\norm{\; B^s(f_i - f_j)\;}^2_{\h}> \eps^2 $ and  

\begin{equation}
\label{est:kullback}
\K(\PP_i, \PP_j) \leq C_{\nu_{*},s}\; R^2 \sigma^{-2}  \paren{\frac{\eps}{R}}^{\frac{2r+1}{r+s}}   \;,
\end{equation}
\item $\log (N_{\eps}-1) \geq \frac{1}{36}\F(2^{\nu_{*}}\left( \frac{\eps}{R}\right)^{\frac{1}{r+s}})$ .
\end{enumerate}
\end{prop}

\begin{proof}[Proof of Proposition \ref{prop:kullback}]
Assume \eiglow \; for any $j \geq j_0$, for some $j_0 \in \NN$.   
Let $max:=\max(28, j_0)$. Choose $\eps < 2^{-\nu_{*}(r+s)}R\; \mu_{max}$ and pick $m = m(\eps):= \F(2^{\nu_*}(\eps R^{-1})^{\frac{1}{r+s}})$.
Note that $m \geq 28$, following from the choice of $\eps$, so Lemma 6.3 in \cite{BlaMuc16} applies.

Let $N_m > 3$ and $\rad_1,\ldots,\rad_{N_m} \in \{-1,+1\}^m$ be given by Lemma 6.3 in \cite{BlaMuc16} and define 
\begin{equation}
\label{gi}
g_i := \frac{\eps}{\sqrt{m}} \sum_{l=m+1}^{2m} \rad_i^{(l-m)} \left(\frac{1}{\eigv_l}\right)^{r+s}  e_l \; .
\end{equation}
We have by the definition of $m$  
\[
\norm{g_i}^2_{\h}  =  \frac{ \eps^2}{m}\sum_{l=m+1}^{2m} \left(\frac{1}{\eigv_l}\right)^{2(r+s)}
\leq \eps^2 \eigv_{2m}^{-2(r+s)} \leq  \eps^2  2^{2\nu_*(r+s)}\eigv_{m}^{-2(r+s)} \leq R^2 \,.\]
For $i = 1,\ldots,N_{m}$ let $f_i:=B^{r}g_i \in \Omega(\nu, r,R)$, with $g_i$ as in $(\ref{gi})$. 
Then  
\[ \norm{\; B^s(f_i - f_j)\;}^2_{\h} \geq \eps^2  \,, \]
as a consequence of Lemma 6.3 in \cite{BlaMuc16}. 
For $i = 1,\ldots,N_{\eps}$, let $\PP_i=\PP_{f_i}$ be defined by (5.6). Then, using the definition of $m$ and (5.7), the Kullback divergence satisfies 
\begin{align*}
\K(\PP_i, \PP_j)  &=  \frac{1}{2\sigma^2} \; \big\| \sqrt{B}(f_i-f_j) \big\|^2_{\h} \leq (2\sigma)^{-2} \eigv_{m+1}^{1-2s} \eps^2 \\
&\leq   2^{\nu_*(1-2s)}  \; (2\sigma^2)^{-1} R^{2}  \paren{\frac{\eps}{R}}^{\frac{1+2r}{r+s}} \,,
\end{align*}
which completes the proof of the first part. 
Finally, again by Lemma 6.3 in \cite{BlaMuc16} and the definition of $m$ 
\[ \log(N_m-1) \geq \frac{m}{36} = \frac{1}{36}\F\left(2^{\nu_*}\;\left( \eps R^{-1}\right)^{\frac{1}{r+s}}\right)   \]
and the proof is complete. 
\end{proof}


\begin{proof}[Proof of Theorem \ref{minimaxlowerrate}:]
Our aim is to apply {Proposition \ref{Fano} in supplementary} and we will check that all required conditions are satisfied. From Proposition \ref{prop:kullback} we deduce that 
for any $\eps$ sufficiently small, there exists $N_{\eps}$ and functions $f_1, \ldots, f_{N_{\eps}}$ satisfying \SC \, fulfilling points 1 and 2\, . The first part of point 1 
gives requirement $(i)$ of {Proposition \ref{Fano} in supplementary}. 
Requirement $(ii)$ follows directly from \eqref{est:kullback} and from point 2 in Proposition \ref{prop:kullback}:

\begin{align*}
\frac{1}{N_{\eps}-1} \sum_{j=1}^{N_{\eps}-1}{\cal K}(\PP^{\otimes n}_j, \PP^{\otimes n}_{N_{\eps}}) &\leq 
n 36 C'_{\nu_{*},s} \; R^2 \sigma^{-2} \; \G\left(2^{\nu_*} \left( \eps R^{-1}\right)^{\frac{1}{r+s}} \right) \log(N_{\eps }-1) \\
& =:  \omega \;\log(N_{\eps }-1)\,,
\end{align*}
with $ C'_{\nu_{*},s} = 2^{-2\nu_*(r+s) - 1} < 1$.  
Define  $\eps :=   2^{-\nu_*}\frac{R}{288}  \G^{-1}( \frac{\sigma^2}{R^2n}    )^{r+s}$\,,
then by Lemma \ref{lem:prop_FG}, the requirements of Proposition {Proposition \ref{Fano} in supplementary} will hold (in particular, 
$\omega < 1/8$) for any $n$ sufficiently large and 
\begin{align*}
\inf_{ f_{\bullet}} \max_{1\leq j\leq N_{\eps}} \PP_j^{\otimes n} \left( \; \big\| B_{\nu}^s ( f_{\bullet} - f_j) \big\|_{\h}  \geq \frac{\eps}{ 2}\; \right) 
&\geq  
       \frac{\sqrt{N_{\eps}-1}}{1 + \sqrt{N_{\eps}-1}} \left( 1-2\omega- \sqrt{\frac{2\omega}{\log{ (N_{\eps}-1})}} \right) > 0 \; . 
\end{align*}
Taking the liminf finishes the proof.  
\end{proof}

%



\bibliography{bibliography}
\bibliographystyle{abbrv}


\end{document}